\documentclass[]{article}
\usepackage{proceed2e}
\usepackage{graphicx,amsmath,amsthm,amssymb}

\title{Semi-Myopic Sensing Plans for Value Optimization}

\author{  {\bf David Tolpin} \\
Computer Science Dept. \\  
Ben-Gurion University \\ 
84105 Beer-Sheva, Israel \\ 
\And
 {\bf Solomon Eyal Shimony}  \\ 
Computer Science Dept. \\  
Ben-Gurion University \\ 
84105 Beer-Sheva, Israel \\ 
}

\newtheorem{thm}{Theorem}
\newtheorem*{dfn}{Definition}

\begin{document} 
 
\maketitle 

\begin{abstract} 
We consider the following sequential decision problem. 
Given a set of items of unknown utility, we need to select one of as high
a utility as possible (``the selection problem'').  Measurements
(possibly noisy) of item values prior to selection are allowed, at a
known cost.  The goal is to optimize the overall sequential decision process
of measurements and selection.

Value of information (VOI) is a well-known scheme for
selecting measurements, but the intractability of the problem typically
leads  to using myopic VOI estimates.
In the selection problem, myopic VOI frequently badly underestimates the
value of information, leading to inferior sensing plans.
We relax the strict myopic assumption into a scheme we term semi-myopic,
providing a spectrum of methods that can improve the performance
of sensing plans. In particular, we propose the efficiently computable method of
``blinkered'' VOI, and examine theoretical bounds for special cases.
Empirical evaluation  of ``blinkered'' VOI in the selection problem
with normally distributed item values shows that is
performs much better than pure myopic VOI.
\end{abstract} 

\section{INTRODUCTION} \label{sec:intro}

Decision-making under uncertainty is a domain with numerous important applications.
Since these problems are intractable in general, special cases are of interest. In this paper,
we examine the selection problem: given a set of
items of unknown utility (but a distribution of which is known),
we need to select an item with as high a utility as possible. 
Measurements (possibly noisy) of item values prior to selection are allowed, at a
known cost.  The problem is to optimize the overall decision process
of measurement and selection. Even with the severe
restrictions imposed by the above setting, this decision problem
is intractable  \cite{RadovilskyOSS}; and yet it is important to be
able to solve, at least approximately, as it has several potential applications, such as
sensor network planning, and oil exploration.

Other settings where this problem is applicable
are in considering which time-consuming deliberation steps to perform
(meta-reasoning \cite{Russell.right}) before selecting an action,
locating a point of high temperature using a limited number
of measurements (with dependencies between locations as in
\cite{Guestrin.graphical}),
%performing costly measurements in order to
%find the best time to hit a target when the system is modeled using a
%stochastic system estimator, such as a
%Kalman filter, and finally, 
and the problem of finding a good set of parameters
for setting up an industrial imaging system.
The latter is actually the original motivation for this research, and
is briefly discussed in section \ref{sec:real-data}.

A widely adopted scheme for selecting measurements, also called sensing actions
in some contexts (or deliberation steps, in the context of meta-reasoning) is based on value
of information (VOI) \cite{Russell.right}.  Computing value of information
is intractable in general, thus both researches and practitioners
often use various forms of myopic VOI estimates \cite{Russell.right} coupled with
greedy search.
Even when not based on solid theoretical guarantees, such estimates lead to
practically efficient solutions in many cases.

However, in a selection problem involving real-valued
items, the main focus of this paper, coupled with the capability
of the system to perform more than one measurement for
each item, the myopic VOI estimate can be shown to badly underestimate
the value of information. This can lead to inferior measurement
sequences, due to the fact that in many
cases no measurement is seen to have a VOI estimate greater than its cost,
due to the myopic approximation.
Our goal is to find a scheme that, while still
efficiently computable, can overcome this limitation of myopic VOI.
We propose the framework of semi-myopic VOI, which includes the myopic
VOI as the simplest special case, but also much more general schemes such as
measurement batching, and exhaustive subset selection at the other extreme.
Within this framework we propose the ``blinkered'' VOI estimate, a variant
of measurement batching, as one that is efficiently computable and
yet performs much better than myopic VOI for the selection problem.

The rest of the paper is organized as follows.
We begin with a formal definition of our version of the
selection problem and other preliminaries. We then examine
a pathological case of myopic VOI, and
present our framework of semi-myopic VOI.
The ``blinkered'' VOI is then defined as scheme within the framework, followed by theoretical
bounds for some simple special cases. Empirical results, comparing
different VOI schemes to blinkered VOI, further support using this
cheme in the selection problem. We conclude with a discussion of
closely related work.

\section {BACKGROUND}

We begin by formally defining the selection problem, followed by
a description of the standard myopic VOI approach for solving it.
 
\subsection {The Optimization Problem} \label {subsec:bg-op}

The selection problem is defined as follows. Given

\begin{itemize}
\item A set $S=\{s_1,\,s_2,\,...,\,s_n \}$ of $n$ items;
\item initial beliefs (probability distribution) about item values
  $Bel(S) = p(s_1=x_1, \,...,\,s_n=x_n)$;
\item utility function $u: {\cal R} \rightarrow {\cal R}$;
\item a cost function $c: {\cal N} \rightarrow {\cal R}$ 
defining the cost of a single measurement of item $i$;
\item a budget or maximum allowed number of measurements $m$;
\item a measurement model, i.e. a probability distribution of observation outcome given
  the true value of each item $p_o(y|x)$;
\end{itemize}

find a policy that maximizes expected net utility $U^{net}$
of the selection --- utility of the selected item less the cost of
the measurements. Although in practice we allow different types
of measurements (i.e. with different cost and measurement error model) we
assume initially for simplicity that all measurements of an item are i.i.d. given
the item value. Thus, if item $s_i$ is chosen after measurement sequence $(M_1,...,M_{m^\prime})$,
and the true value of $s_i$ is $x_i$,  the net result $U^{net}$ is:

\begin{equation}\label{eq:bg-netresult}
U^{net} = u(x_i)-\sum_{j=1}^{m^\prime} c(M_j) 
\end{equation}

We assume that the posterior beliefs $p^+(s_1=x_1, s_2=x_2,\,...,\,s_n=x_n|y)$;
and the marginal posterior beliefs
$p^+(s_i=x|y)$ about an item value can be computed efficiently.
More specifically, we usually represent the distribution using a structured probability
model, such as a Bayes network or Markov network. The assumption is that
either the structure or distribution of the network is such that belief updating is easy, 
that the network is sufficiently small, or that the network is such that
an approximate and efficient belief-updating
algorithm (such as loopy belief updating) provides a good approximation.
Observe that this assumption does not make the selection problem tractable,
as even in a chain-structured network the selection problem is NP-hard \cite{RadovilskyOSS}.
In fact, even when assuming that the beliefs about items are independent
(as we will do for much of the sequel), the problem is still hard.

\subsection {Limited Rationality Approach} \label {subseq:bg-lra}

In its most general setting, the selection problem can be modeled as
a (continuous state) indefinite-horizon POMDP \cite{HansenPOMDP}, which is badly intractable.
Following \cite{Russell.right}, we thus use a
greedy scheme that:

\begin{itemize}
\item Chooses at each step a measurement with the greatest value of information (VOI), performing
belief updating after getting the respective observation,
\item stops when no measurement with a positive VOI exists, and
\item selects item $s_\alpha$ with the greatest expected utility:
\begin{equation}\label{eq:bg-eu}
E(U_\alpha) = \int\limits_{-\infty}^\infty p_\alpha(x)u(x)dx
\end{equation}
\end{itemize}

VOI of a measurement $M_j$ is defined as follows: denote by $E(U_i^{j+})$
the expected net utility of item $s_i$ after measurement $j$ and a subsequent
optimal measurement plan. Let $s_{\alpha^-}$ be the item that currently has the
greatest expected net utility  $E(U_{\alpha^-})$. Likewise, let
$s_{\alpha^j+}$ be the item with the greatest expected utility $E(U_{\alpha^{j+}}^{j+})$ 
{\em after} a measurement plan beginning with observation $j$. Then:

\begin{equation}\label{eq:bg-voi}
VOI(M_j) = E\left(U_{\alpha^{j+}}^{j+}\right)-E\left(U_{\alpha^-}^{j+}\right)
\end{equation}

\subsection {Myopic Value of Information Estimate} \label {subsec:bg-mvi}

Computing value of information of a measurement is intractable, and
is thus usually estimated instead under the assumptions of {\it meta-greediness}
and {\it single-step}. Under these assumptions, the myopic scheme
considers only one measurement step, and
ignores value of later measurements. A measurement
step can consist of a single measurement or of a fixed number thereof
with no deliberation between the measurements.

A measurement is  beneficial only if it changes which item
appears to have the greatest estimated expected utility. 
For items that are mutually independent (essentially the {\em subtree independence}
assumption of  \cite{Russell.right}), a measurement only affects beliefs
about the measured item. If the measured item does not seem to be the best now, but can
become better than the current best item $\alpha$ when the belief is
updated, the benefit in this case is:

\begin{equation}\label{eq:mvi-benefit-alpha}
B_i(y)=\max \left(\int\limits_{-\infty}^\infty u(x)p_i^+(x|y)dx-E(U_\alpha),0\right)
\end{equation}

If the measured item $\alpha $ can become worse
  than the next-to-best item $\beta$, the benefit is:

\begin{equation}\label{eq:mvi-benefit-beta}
B_i(y)=\max \left(E(U_\beta)-\int\limits_{-\infty}^\infty u(x)p_i^+(x|y)dx,0 \right)
\end{equation}
where $y$ is the observed outcome, and $p_i^+(x|y)$ is the posterior belief.

For these two cases, the myopic VOI estimate $MVI$ of observing the
$i$th item at step $j$ of the algorithm is:

\begin{equation} \label{eq:bg-mvi}
MVI_i = \int\limits_{-\infty}^{\infty}p_i^-(x)\int\limits_{-\infty}^{\infty}B_i(y)p_o(y|x)dydx-c(j)
\end{equation}

%\subsection {Normally Distributed Beliefs and Observations} \label{subsec:bg-nbel}
%
%Beliefs and noisy measurements are commonly modelled by normal
%distributions. We too will use normal distributions for examples and
%experiments. For convenience, we provide here formulas  for
%posterior beliefs and belief distributions in this case. The formulas
%follow directly from belief propagation equations for Gaussian Bayesian
%networks in \cite{Pearl.PRIS}.

%Denote by $N(x;\mu, \sigma^2)$ the normal probability distribution of $x$ with
%mean $\mu$ and variance $\sigma^2$. If the prior belief about an item's
%value is $N(x;\mu_-, \sigma_-^2)$  and observation outcomes are normally
%distributed around the true value as $N(y; 0, \sigma_o^2)$, the the posterior
%belief is 

%\begin{equation}\label {eq:nbel-pb}
%p_+(x|y)=N\left(x; \frac {\mu_-\sigma_o^2+y\sigma_-^2} {\sigma_o^2+\sigma_-^2},\frac {\sigma_o^2\sigma_-^2} {\sigma_o^2+\sigma_-^2}\right)
%\end{equation}

%The algorithm decides whether and where to observe based on its belief
%about the posterior belief in each location. The mean of the posterior
%belief is a linear transformation of a normally distributed normal
%variable and thus is itself normally distributed.

%\begin{equation}\label{eq:nbel-pbd}
%p_\mu(x)=N\left(x;\mu_-, \frac {\sigma_-^4} {\sigma_o^2+\sigma_-^2}\right)
%\end{equation}

%As expected, $\sigma_-^2=\sigma_\mu^2+\sigma_+^2$.

\subsection {Myopic Scheme: Shortcomings} \label{sec:mvilim}

The decisions the myopic scheme makes at each step are:
which measurement is the most valuable, and
whether the most valuable measurement has a positive value.
The simplifying assumptions are justified when they
lead to decisions sufficiently close to optimal in
terms of the performance measure, the net utility.

The first decision controls search ``direction''.  When it is wrong,
a non-optimal item is measured, and thus more
measurements are done before arriving at a final decision. The net utility decreases due
to the higher costs.
The second decision determines when the algorithm terminates, and can
be erroneously made too late or too early. Made too late, it leads to
the same deficiency as above: the measurement cost is too
high. Made too early, it causes early and potentially incorrect item selection,
due to wrong beliefs. The net utility
decreases because the item's utility is low.

In terms of value of information, the assumptions lead to correct
decisions when the value of information is negative for every sequence of steps,
or if there exists a measurement that according to the
meta-greedy approach has the greatest (positive) VOI estimate,
and is the head of an optional measurement plan.
These criteria are related to the notion of {\it non-increasing returns};
the assumptions are based on an implicit hypothesis that
the intrinsic value grows more slowly than the cost. When the hypothesis is
correct, the assumptions should work well; otherwise, the myopic scheme
either gets stuck or goes the wrong way.

It is commonly believed that many processes exhibit diminishing
returns; the {\it law of diminishing returns} is considered a universal
law in economics \cite{Johns.economics}. However, this only holds
asymptotically: while it is often true that starting with some point
in time the returns never grow, until that point they can alternate
between increases and decreases.
Sigmoid-shaped returns were discovered in marketing \cite{Johansson.s-curve}.
As experimental results show \cite{Zilberstein.sensing}, they are not uncommon in sensing
and planning. In such cases, an approach that can deal with increasing
returns must be used.

\subsection {Pathological Example} \label{subsec:mvilim-ex}

Let us examine a simple example where the myopic estimate behaves poorly:

\begin {itemize}
\item $S$ is a set of two items, $s_1$ and $s_2$, where
the value of $s_1$ is known exactly, $x_1=1$;
\item the prior belief about the value of $s_2$ is a normal distribution $p_2(x)=N(x;0,1)$;
\item the observation variance is $\sigma_o=5$;
\item the observation cost is constant, and chosen so that the net value estimate of
  a two-observation step is zero: $c(j)=\frac {MVI_2^2} 2 \approx 0.00144$;
\item the utility is a step function:

\[u(x) = \left\{ 
\begin{array}{l l}
  0 & \quad \mbox{if $x<1$}\\
  0.5 & \quad \mbox{if $x=1$}\\
  1 & \quad \mbox{if $x>1$}\\
\end{array} \right. \]

\end {itemize}

The plot in Figure \ref{fig:mvilim-value-cost} is computed according
to belief update formulas for normally distributed beliefs
and presents dependency of the intrinsic value estimate on the number
of observations in a single step. The straight line corresponds to the
measurement costs.

\begin{figure}[h]
\includegraphics[scale=0.75,trim=0pt 15pt 0pt 45pt,clip]{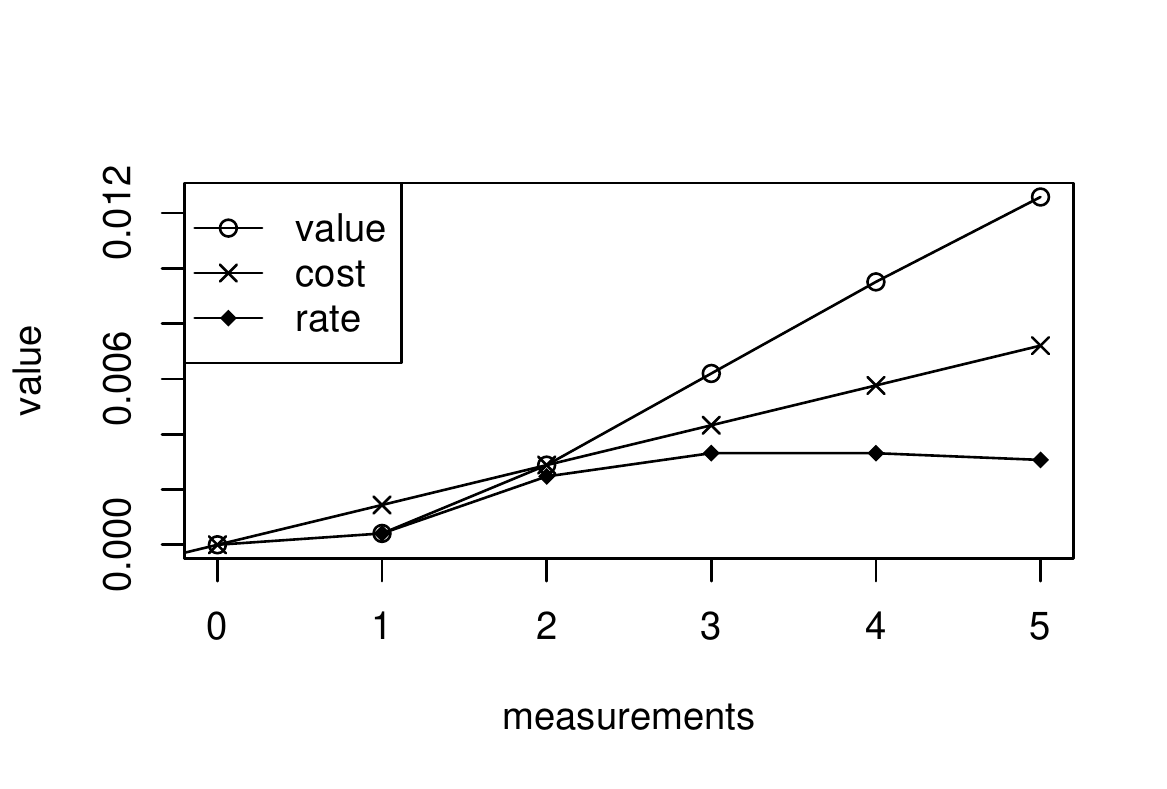}
\caption{Intrinsic value and measurement cost} 
\label{fig:mvilim-value-cost}
\end{figure} 

Under these conditions, the algorithm with one measurement per step
will terminate without gathering evidence because the value estimate
of the first step is negative, and will return item $s_1$ as
best. However, observing $s_2$ several times in a row has
a positive value, and the updated expected utility of $s_2$ can eventually
become greater than $u(s_1)$.
Figure \ref{fig:mvilim-value-cost}  also shows the intrinsic value growth rate
as a function of the number of measurements: it increases up to
a maximum at 3 measurements, and then goes down. Apparently,
the myopic scheme does not ``see'' as far as the initial increase.

\section{SEMI-MYOPIC VOI ESTIMATES} \label{sec:semimvi}

The above pathological example was inspired by a phenomenon we
encountered in a real-world problem, optimizing parameter in setups of imaging machines.
On data with varying prior beliefs, the myopic scheme almost never measured  an item
with high prior variance even when it was likely to become the best item
after a sufficient number of measurements.
 
%So far, we have discovered that the algorithm driven by the myopic VOI
%estimate can terminate prematurely, and that it happens because the
%intrinsic value can grow faster than the time cost for certain
%combinations of parameters. To make the algorithm suitable for the
%problem, we can either find a way to bypass such combinations, or to
%use a different estimate.

Keeping the complexity manageable (the number of possible sensing plans
over continuous variables is uncountably infinite, in addition to being multi-dimensional)
while overcoming the premature termination is the basis for the
{\em semi-myopic} value of information estimate. Consider the set of all
possible measurement actions ${\cal M}$. Let $C$ be a constraint over
sets of measurements from ${\cal M}$. In the semi-myopic framework, we consider all
possible subsets $B$ of measurents from ${\cal M}$ that obey the constraint $C$,
and for each such subset $B$ compute a ``batch'' VOI that assumes that
all measurements in $B$ are made, followed by a decision (selection
of an item in our case). Then, the batch $B*$ with best estimated VOI is
chosen. Once the best batch $B*$ is chosen, there are still several options:
\begin{enumerate}
\item Actually do all the measurements in $B*$.
\item Attempt to optimize $B*$ into some form of conditional plan
of measurements and the resulting observations.
\item Perform the best measurement in $B*$.
\end{enumerate}

In all cases, after measurements are made, the
selection is repeated, until no batch has a positive net VOI,
at which point the algorithm {\em terminates} and selects an item.
Although we have experimented with option 1 for comparative purposes,
we did not further pursue it as empirical performance
was poor, and opted for option 3 in this paper\footnote{Option 2 is intractable in general, and
while limited efficient optimization may be possible, this issue is beyond the scope
of this paper.}. 
Observe that the constraint $C$ is
crucial. For an empty constraint (called the {\em exhaustive} scheme),
all possible measurement sets are considered.
Note that this has an exponential computational cost, while {\em still} not guaranteeing
optimality.
At the other extreme, the constraint is that only singleton
sets be allowed,
resulting in the greedy single-step assumption, which we call the {\em pure myopic} scheme.

The myopic estimate can be extended to the case when a
single step consists of $k\ge 1$ measurements of a single item
$s_i$. We denote the estimate of such a $k$-measurements step by
$MVI_i^k$. Our main contribution is thus the case where the constraint is that all measurements
be for the same item, which we call the {\em blinkered} scheme.
Yet another scheme we examine is where the constraint allows at most
one measurement for each item (thus allowing from zero to
$n$ measurements in a batch), called the {\em omni-myopic} scheme.

\subsection {Blinkered Value of Information} \label{subsec:bvi}

As stated above, the blinkered scheme considers 
sets of measurements that are all for the same item; this constitutes unlimited lookahead,
but along a single ``direction'', as if we ``had our blinkers on''.
That is, the VOI is estimated for any number of
independent observations of a single item. 
Although this scheme has a significant
computational overhead over pure-myopic, the factor is only
linear in the maximum budget. \footnote{This complexity assumes either normal distributions, or some other
distribution that has compact statistics. For general distributions,
sets of observations may provide information beyond the mean and variance,
and the resources required to compute VOI may even be exponential in the number of measurements.}
We define the ``blinkered'' value of information as:

\begin{equation}\label{eq:bvi} BVI=\max_k MVI^k \end{equation}

Driven by this estimate, the blinkered scheme selects a single measurement of
the item where {\em some} number of measurements gain the greatest VOI.
A single step is expected not to achieve the value, but to be
just the first one in the right direction. Thus, the estimate relaxes
the {\it single-step} assumption, while still underestimating the
value of information.

For finite budget $m$ the time to compute the estimate $T_{BVI}$ is:
$T_{BVI}=O\left(T_{MVI}m\right)$.  
If $MVI^k$ is a unimodal function of $k$, which can be shown for some
forms of distributions and utility functions,
the time is only logarithmic in $m$ using bisection search.
%Indeed, using the
%bisection method to find $k$ for which $MVI^{k+1}-MVI^k$ changes sign
%yields the logarithmic time.

%Using the blinkered scheme,
%susceptibility to premature termination of measurements is reduced
%(resulting in better item selection)
%without noticeable computational overhead, as shown empirically
%in Section \ref{sec:experiments}. A formal analysis of
%these schemes is difficult, but can be done for special cases.

\subsection {Theoretical Bounds} \label {subsec:bvi-bounds}

We establish bounds on the blinkered scheme
for two special cases,
beginning with the termination condition for the case of only 2 items.

%\subsubsection {Single Measurable Item} \label {subsec:bvi-bounds-single}

\begin{thm}\label{th:bound-single} Let $S=\{ s_1, s_2\} $, where the value of $s_1$ 
is known exactly. Let $m_b$ be the remaining budget of measurements when
the blinkered scheme terminates, and $C$ be the (positive) cost
of each measurement. Then the (expected) value of information that
can be achieved by an optimal policy from this point onwards is at most $m_bC$.
\end{thm}

\begin{proof}[Proof] 
The intrinsic value of information of the remaining budget when the blinkered
scheme terminates  is $V_b^{int} \leq m_bC$, since otherwise it would not have
terminated. Since there is only one type of measurement, the intrinsic value of
information $V_o^{int}$ achieved by an optimal policy must be at most equal
to making all the measurements, thus $V_o^{int} \leq V_b^{int} \leq m_bC $.
Since measurement costs are positive, the net value of information  $V_o^{net}$ achieved
by the optimal policy must therefore also be at most $m_bC $.
\end {proof}

This bound is asymptotically tight. This can be shown by having a measurement
model with dependencies such that the first measurements do not change the
expected utilities of the item, but provide information on whether it is worthwhile
to perform additional measurements.

% Consider the following problem
%instance: $s_1$ is known to have utility 0, $s_2$
%can have utility $\{ - \frac {m_b C}p , \frac {m_b C} p\}$, with uniform prior.
%After the first measurement, 

%\begin{itemize}
%\item the probability of reward $\frac {C(m_b)} p$ after the first measurement
%  is $p$ for some $0\le p\le1$;
%\item the probability of reward $0$ after the first observation is $1-p$;
%\item the reward is given at the end of the budget;
%\item an optimal algorithm decides with confidence after the first observation
%  whether it will get the reward.
%\end{itemize}

%The net value of the blinkered algorithm is 0 and it terminates. Meanwhile,
%the net value of an optimal algorithm may approach $C(m_b)$: 
%$$ \lim_{\substack{p\to 0\\c(m-m_b+1)\to 0}}V_o^{net}=C(m_b) $$

%\subsubsection {Multiple Items} \label{subsec:bvi-bounds-multiple}

The second bound is related
to a common case of a finite budget and free measurements. It
provides certain performance guarantees for the case when 
dependencies between items are sufficiently weak.

\begin{dfn} \label{dfn:mutually-submodular} Measurements of two items $s_1$, $s_2$ are 
mutually submodular if, given sets of measurements
of each item $M_1$ and $M_2$, the intrinsic value of information of the
union of the set is not greater than the sum of the intrinsic values of
each of the sets, i.e.: $V^{int}(M_1 \cup M_2) \le V^{int}(M_1) + V^{int}(M_2)$
\end{dfn}

\begin{thm} \label{th:bound-multiple} For a set of $n$ items,
measurement cost $C=0$, and a finite budget of $m$ measurements, if
measurements of every two items are mutually submodular, the
value of information collected by the blinkered scheme is no more than a factor of $n$
worse than the value of information collected by an optimal measurement plan.
\end {thm}

\begin{proof} Since the measurement cost is $0$, $V^{net}=V^{int}$,
we omit the superscript in the proof.  Expected value of
information cannot decrease by making additional measurements,
therefore the value of any set of measurements containing the set
of measurements in an optimal plan is at least as high as the
value $V_o$ of an optimal plan. Consider an exhaustive plan
containing $m$ measurements of each of the $n$ items, $mn$ measurements
total with value $V_e$. The exhaustive plan contains all measurements that can
be made by optimal plan within the budget, thus $V_o\le V_e$.

Let $s_i$ be the item with the highest blinkered value for $m$
measurements, denote its value by 
$V_{bimax}=\max_i V_{b,i}$. Since measurements of different items are mutually
submodular, $V_{e} \le nV_{bimax}$, 
and thus $V_{bimax}\ge \frac {V_o} n$.

The blinkered scheme selects at every step a measurement from a plan
with value of information which is at least as high as the
value of measurements of $s_i$ for the rest of the
budget. Thus, its value of information $V_b \ge V_{bimax}
\ge \frac {V_o} n$.
\end{proof}

The bound is asymptotically tight. Construct a problem instance
as follows: $n$ items, with expected values of information $v_j(i)$,
for $i$ measurements, respectively, defined below.
Value of information of a combination of the measurements
is the sum of the values for each item, $v(i)=\sum_{j=1}^n v_j(i_j)$, with the following
value of information functions:

$$v_1(i) = \sqrt[k] {\frac i m}$$

\[v_{j>1}(i) = \left\{ 
\begin{array}{l l}
  0 & \quad \mbox{if $i<\frac m n$}\\
  \sqrt[k] {\frac 1 n} & \quad \mbox{otherwise}\\
\end{array} \right. \]

Here, the optimal policy is to measure each item $\frac m n$ times. The resulting
value of information for $m$ measurements is:

\[
V_o = v_1\left(\frac m n\right)+(n-1)v_{j>1}\left(\frac m n\right) = n \sqrt[k] {\frac 1 n}
\]
\[\lim_{k\to\infty}n \sqrt[k] {\frac 1 n} = n
\]

But the blinkered algorithm will always choose the first item with $V_b = v_1(m) = 1$.

\section {EMPIRICAL EVALUATION} \label {sec:exprm}

It is rather difficult to perform informative experiments
on a real setting of the selection problem. Therefore,
other than one case coming from a parameter selection application,
empirical results in this paper are for artificially generated
data.

%\subsection {Maximum Time Costs} \label {subsec:exprm-maxtc}

%To illustrate the dependency between the intrinsic value and the number
%of observations, we extended the plot from the example to a longer
%budget, and superimposed it with the plots for greatest time costs
%(when measurement cost is constant) which still allow observations for the
%myopic with a single observation per step (the dotted line), myopic with
%two observations per step (the dashed line) and blinkered estimates (the 
%solid line tangent to the value curve).

%\begin{figure}[ht] 
%\label{fig:exprm-value-timecost}
%\includegraphics[scale=0.65,trim=0pt 0pt 0pt 50pt,clip]{fig-exprm-value-timecost.pdf}
%\caption{Intrinsic value and measurement costs} 
%\end{figure} 

%One can see from the plot (Figure \ref{fig:exprm-value-timecost}), that because the returns are
%still asymptotically diminishing, there is a high enough measurement cost
%which prevents the algorithm from committing the observation under any
%value estimate. Meanwhile, there is a range of measurement costs for which
%the blinkered estimate causes further observations, while the myopic
%estimate leads to termination of the search.

\subsection {Simulated Runs} \label {subsec:exprm-runs}

The first set of experiments is for independent normally distributed
items. For a given number of items, we randomly generate their exact
values and initial prior beliefs. Then, for a range of measurement costs,
budgets, and observation precisions, we run the search, randomly
generating observation outcomes according to the exact values of the
items and the measurement model. The performance measure  is
the regret - the difference between the utility of the best item and the utility
of the selected item, taking into account the measurement costs.
We examine the result of using the
blinkered scheme vs. other semi-myopic schemes.

The first comparison is the difference in regret between the myopic and blinkered
schemes, done for 2 items, one of which has a known value
(Table \ref{tbl:n2-t5}). A positive values in the cells, indicates
an improvement due to the blinkered estimate. Note that
the absolute value is bounded by 0.5, the difference in the
utility of the exactly known item and the extremal utility.
%Table \label{tbl:n4-t10} repeats the experiments for 4 items.

\begin{table}[h] 
\begin{center} 
\begin{tabular}{l|r r r r}
$\sigma_o ~~ \backslash ~~C$ & 0.0005  &  0.0010 & 0.0015 & 0.0020 \\ \hline
3        & 0.0147 & 0.0156 & 0.0199 & 0.2648 \\
4        & 0.0619 & 0.2324 & 0.2978 & 0.2137 \\
5        & 0.2526 & 0.2322 & 0.1729 & 0.1776 \\
6        & 0.1975 & 0.1762 & 0.1466 & 0.0000 \\
\end{tabular} 
\caption{2 items, 5 measurement budget} 
\label{tbl:n2-t5} 
\end{center} 
\end{table} 

\begin{table}[h] 
\begin{center} 
\begin{tabular}{l|r r r r}
$\sigma_o ~~ \backslash ~~C$ & 0.0005  &  0.0010 & 0.0015 & 0.0020 \\ \hline
3            & 0.0113 & -0.00459 & -0.0024 & 0.4352 \\
4            & 0.0374 &  0.43435 &  0.4184 & 0.3902 \\
5            & 0.4060 &  0.40004 &  0.3534 & 0.3599 \\
6            & 0.4082 &  0.37804 &  0.3337 & 0.0000 \\
\end{tabular} 
\caption{4 items, 10 measurement budget} 
\label{tbl:n4-t10} 
\end{center} 
\end{table} 

Averaged over 100 runs of the experiment, the difference is
significantly positive for most combinations of the parameters. In the
first experiment (Table \ref{tbl:n2-t5}), the average regret of the myopic scheme compared
to the blinkered scheme is 0.15 with standard deviation 0.1.  In the
second experiment (Table \ref{tbl:n4-t10}), the regret is 0.27 with standard deviation 0.19.

\subsubsection {Other Semi-Myopic Estimates}

In this set of experiments,
we compare three semi-myopic schemes:  blinkered, omni-myopic, and exhaustive. 
All schemes were run on a set of 5 items with a 10 measurement
budget.
% and a 5 measurement horizon. 
The results show that while blinkered is significantly better than pure myopic
(Table \ref{tbl:planhead-mb}), exhaustive is only marginally
better than blinkered (Table \ref{tbl:planhead-be}),
even though it requires evaluating an exponential number
of sets of measurements. Another semi-myopic scheme,
omni-myopic, is only marginally better than myopic (Table \ref{tbl:planhead-mo}).

\begin{table}[h] 
\begin{center} 
\begin{tabular}{l|r r r r}
$\sigma_o ~~\backslash ~~C$ & 0.0005  &  0.0010 & 0.0015 & 0.0020 \\ \hline
3.0 & -0.1477 & 0.0946 & 0.1889 & 0.2807 \\
4.0 & 0.0006 & 0.0382 & 0.4045 & 0.4180 \\
5.0 & 0.2300 & 0.3954 & 0.2925 & 0.3222 \\
6.0 & 0.0494 & 0.2374 & 0.1452 & 0.2402 \\
\end{tabular} 
\caption{myopic vs. blinkered} 
\label{tbl:planhead-mb} 
\end{center} 
\end{table} 

\begin{table}[h] 
\begin{center} 
\begin{tabular}{l|r r r r}
$\sigma_o ~~ \backslash ~~C$ & 0.0005  &  0.0010 & 0.0015 & 0.0020 \\ \hline
3.0 & 0.0218 & 0.0307 & 0.1146 & -0.1044 \\
3.0 & -0.0502 & 0.0703 & 0.0598 & 0.0022 \\
3.0 & 0.0940 & -0.1508 & -0.0865 & 0.1432 \\
3.0 & 0.0146 & 0.1485 & -0.0505 & 0.2068 \\
\end{tabular} 
\caption{blinkered vs. exhaustive} 
\label{tbl:planhead-be} 
\end{center} 
\end{table} 

\begin{table}[h] 
\begin{center} 
\begin{tabular}{l|r r r r}
$\sigma_o ~~\backslash ~~C$ & 0.0005  &  0.0010 & 0.0015 & 0.0020 \\ \hline
3.0 & -0.0781 & -0.0167 & 0.0391 & 0.0125 \\
4.0 & 0.0000 & 0.0974 & 0.1848 & -0.0982 \\
5.0 & 0.0609 & -0.0002 & 0.0000 & 0.0000 \\
6.0 & 0.1272 & 0.0000 & 0.0000 & 0.0000 \\
\end{tabular} 
\caption{myopic vs. omni-myopic} 
\label{tbl:planhead-mo} 
\end{center} 
\end{table} 

\subsubsection {Dependencies between Items}

When the values of the items are linearly dependent, e.g. when:
$x_i = x_{i-1} + w$ with $w$ being a random variable distributed as $N(0,\sigma_w^2)$,
the VOI of series of observations of several items
can be greater than the sum of VOI of each observation. We examine the influence
of dependencies on the relative quality of the blinkered and omni-myopic schemes.

When there are no dependencies, i.e. $\frac {\sigma_o^2} {\sigma_w^2} = 0$, the blinkered
scheme is significantly better. But as $\sigma_w$  decreases,
the omni-myomic estimate performs better. Figure \ref{fig:exprm-dependency} shows the difference
between achieved utility of the blinkered and the myopic schemes with
dependencies. The experiment was run on a set of 5 items, with the prior estimate
$N(0,1)$, measurement precision $\sigma_o^2=4$, measurement cost $C=0.002$ and a budget
of 10 measurements. The results are averaged over 100 runs of the experiment.

\begin{figure}[ht] 
\includegraphics[scale=0.65,trim=0pt 15pt 0pt 50pt,clip]{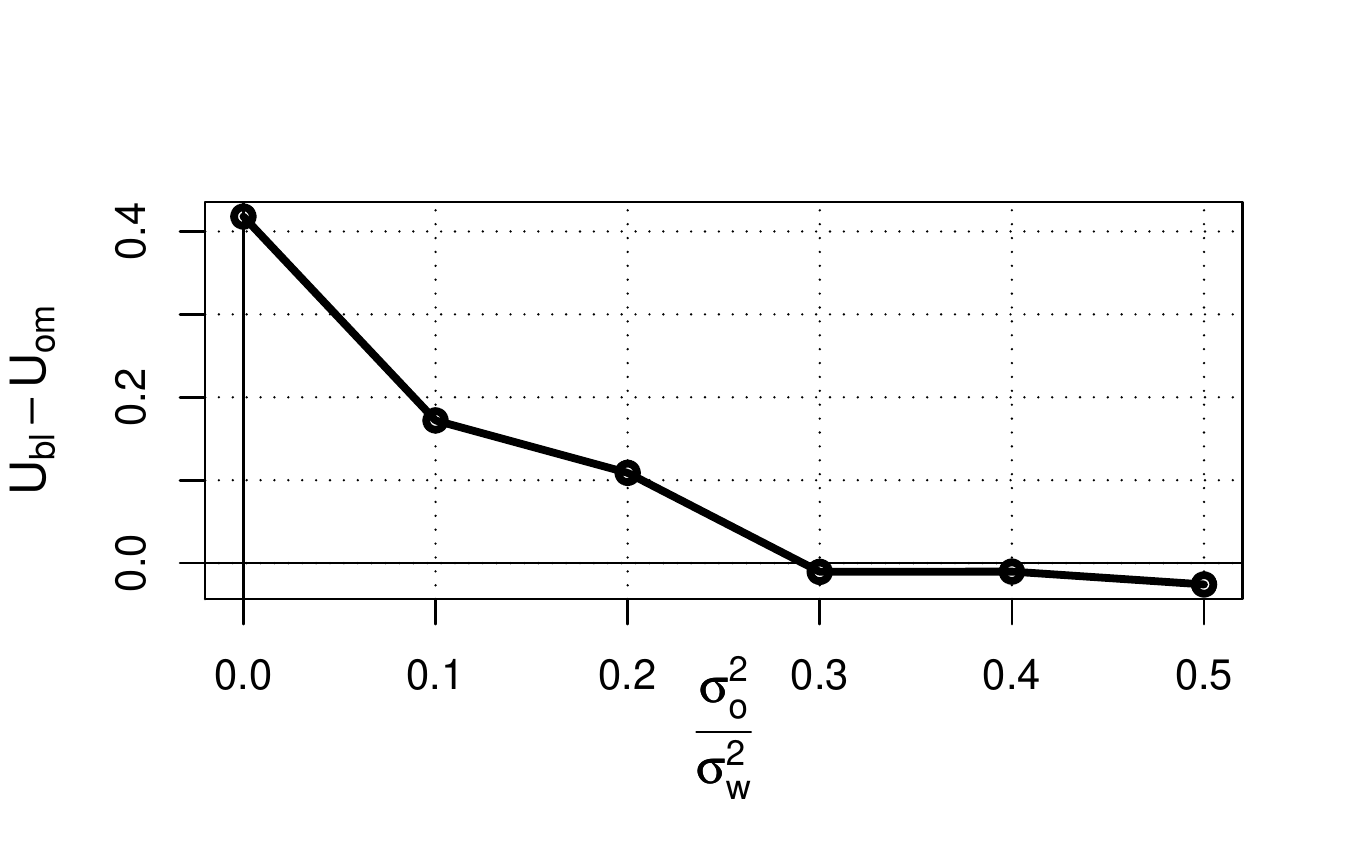}
\caption{Influence of dependencies} 
\label{fig:exprm-dependency}
\end{figure} 

In the absence of dependencies, the omni-myopic algorithm does not perform measurements
and chooses an item at random, thus performing poorly. 
As the dependencies become stronger, the omni-myopic
scheme collects evidence and eventually outperforms the blinkered scheme.
In the experiment, the omni-myopic scheme begins to outperform the blinkered scheme
when dependencies between the items are roughly half as strong as the 
measurement accuracy.
The experimental results thus encourage the use of the blinkered value of
information estimate in problems with increasing returns for certain combinations of parameters and
weak dependencies between the items.

\subsection{Results on Real Data}\label{sec:real-data}

Due to lack of space, we only outline some main aspects of an additional application of the 
selection problem -- parameter optimization for imaging
machines, which has items arranged as points on a multi-dimensional grid, with grid-structured
Markov dependencies. 
In this application one dimension was ``filter color'', and
another dimension was ``focal length index''. The utility of an item
is based on features observed in each image, and we examine results
of one case along just the focal length index dimension.
The utility function is a hyperbolic
tangent of the measured features. We assumed that ``items values'' were normally distributed,
and the dependencies were roughly estimated from the data, measurement variance
$\sigma_o^2 \approx 0.1$ and $\sigma_w^2 \approx 0.2$. 

\begin{figure}[ht] 
\includegraphics[scale=0.65,trim=0pt 15pt 0pt 50pt,clip]{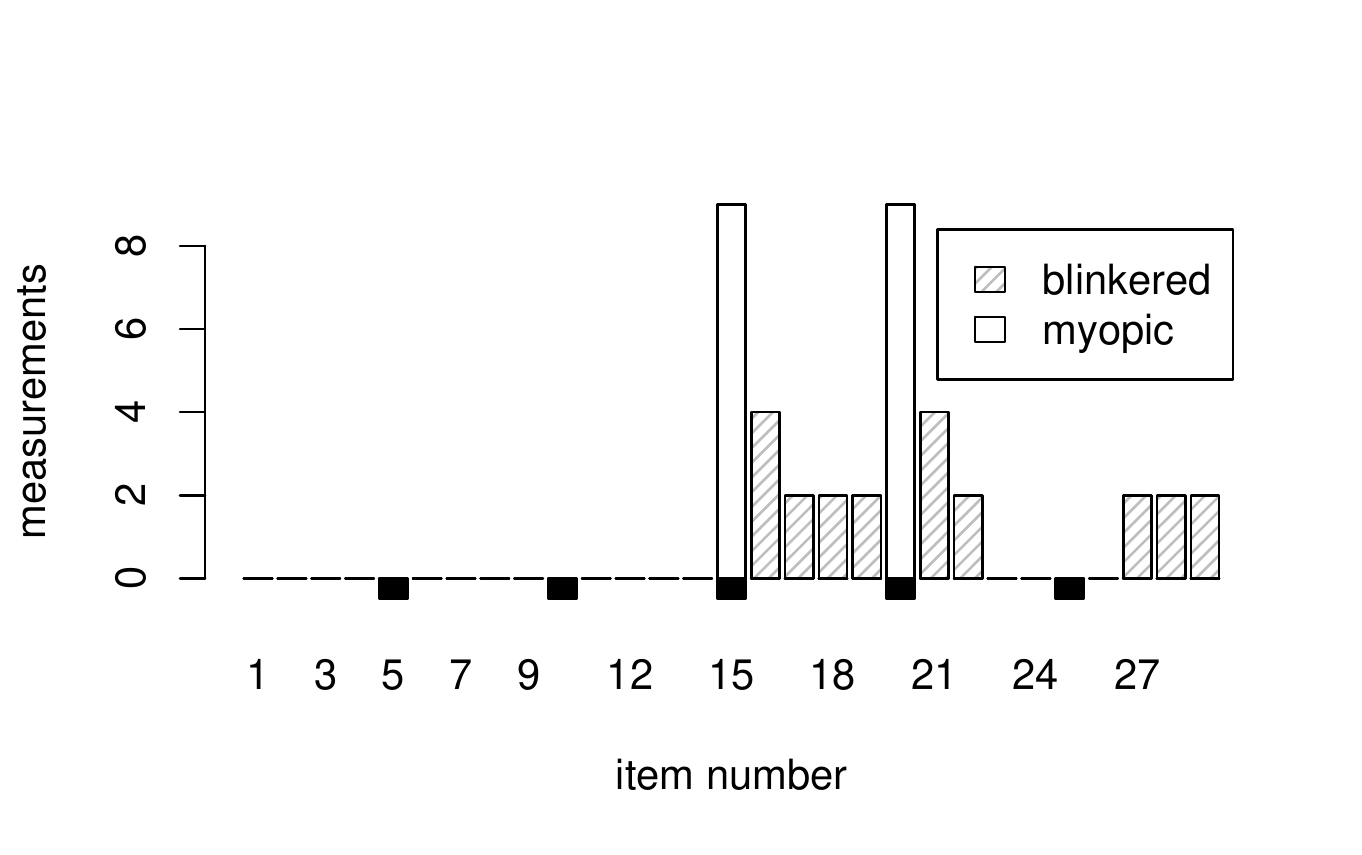}
\caption{Blinkered vs Myopic Measurements} 
\label{fig:real}
\end{figure} 

Figure \ref{fig:real} shows a summary of measurements made the blinkered scheme
vs. pure myopic, where initial beliefs (for most items -  variance was approximately 0.8)
due to some previous measurements
resulted in prior knowledge (smaller variance in beliefs: approximately 0.05) for
focal length indices marked with black boxes. The pure myopic scheme measured only items
with small variances, and eventually picked an inferior focus length index. The
blinkered scheme performed different measurements, ending up in selecting
the optimal index.

These results are for one typical data set. Unfortunately for this problem,
in experiments on real data, the result set is of necessity rather sparse, as
it is difficult to map the space of possibilities as was done for simulated data. 
Such an exploration would require us to predict, e.g. what would have happened
had the item value been different? What would have been the result had
we performed a measurement for (some unmeasured) item? Although the latter question
was handled in our system by physically performing numerous measurements
for all items, the former question is much more difficult to handle.

%\section {DISCUSSION} \label {sec:dis}

\section{RELATED WORK} \label {subsec:dis-related}

%If the ``law of diminishing returns'' is correct asymptotically, then
%the returns tend to be increasing when the observation variance is
%high and the intrinsic value a single observation is low. In this case,
%merging several observations into a single step can ensure
%non-increasing returns. However, determining the number of
%observations exactly is as hard as finding the exact solution, while
%an optimal solution may not be among those with the repeating
%observations.

Limited rationality, a model of deliberation based on value of utility
revision and deliberation time cost was introduced in
\cite{Russell.right}.  Notions of value of
computation and its estimate were defined, as well as the class of meta-greedy
algorithms and simplifying assumptions under which the algorithms
are applicable. The theory of bounded optimality, on which the
approach is based, is further developed in \cite{Russell.bounded}.
\cite{Zilberstein.PHD} employs limited rationality techniques to analyze
any-time algorithms and proves optimality of myopic algorithm
monitoring under assumptions about the class of value and time cost
functions.

\cite{Guestrin.submodular} consider a greedy algorithm for observation
selection based on value of observation. They show that when values
of measurements for different items are mutually submodular and the
measurement cost is fixed, the greedy algorithm is nearly optimal.
The assumptions are inapplicable in our domain, necessitating an
extension of the pure greedy approach in our case.
% The
%authors describe applications of the algorithm to various problems,
%such as fault diagnosis, robotic explorations, minimizing human
%attention and others.

In \cite{Heckerman.nonmyopic}, a case of discrete Bayesian
networks with a single decision node is analyzed. The authors propose
to consider subsequences of observations in the descending order of
their myopic value estimates. If any such subsequence has non-negative
value estimate, then the computation with the greatest myopic estimate
is chosen.  However, this approach 
always chooses a measurement for the myopically best item, and
when applied to the selection problem either looks at sequences of 
measurements on a single item
with the greatest myopic value estimate,
or, if sequences with one measurement per item are
considered, does not provide an improvement over the myopic estimate, for
our pathological example. Still, in may cases their scheme shows an improvement in
performance. \cite{Liao.nonmyopic} describes and
experimentally analyzes an algorithm for influence diagrams based on
a non-myopic VOI estimate.

Multi-armed bandits \cite{Vermorel.bandits} bear similarity
to the measurement selection problem, in particular, when the reward
distribution is continuous and unknown. Some of the algorithms,
e.g. POKER (Price of Knowledge and Estimated Reward) employ the notion
of value of information. However, most solutions concentrate on exploitation
of particular features of the value function, such as linear dependence of
reward from pushing a lever on the time left, and  do not facilitate
generalization. On the other hand, achievements in limited rationality
techniques should be helpful in development of improved solutions in this
domain.

\section {CONCLUSION} \label {sec:further}

We have introduced a new ``semi-myopic'' value of information
framework. An instance of semi-myopic scheme, called the blinkered scheme,
was introduced, and demonstrated to have positive impact on solving the selection
problem. Theoretical analysis of special cases provides some insights.
Empirical evaluation of the blinkered scheme on simulated data
shows that it is promising both for independent and for weakly dependent items.
A limited evaluation of an actual application also indicates that the
blinkered scheme is useful.

Still, properties of the estimate have been
investigated only partially for the dependent case,
which is of more practical importance.
In particular, when, due to sufficiently strong dependencies,
observations in different locations are not mutually submodular, the
blinkered estimate alone may not prevent premature termination of the
measurement plan, and its combination with the approach proposed in
\cite{Heckerman.nonmyopic} may be worthwhile.

During the algorithm analysis, several assumptions have been made about
the shape of utility functions and belief distributions. 
Certain special cases, such as normally distributed beliefs and convex utility
functions, are frequently met in applications and may lead to stronger
bounds and discovery of additional features of semi-myopic schemes.

An important application area of the selection problem, parameter optimization,
has items arranged as points on a multi-dimensional grid, with grid-structured Markov
dependencies. This special case has been partially investigated and
requires future work. Extending this case to points on a continuous grid
should also have numerous practical applications.

%Greedy algorithms based on either the myopic or the blinkered estimate
%can be viewed as partial cases of a more general approach. In that
%approach, some of possible observation subsets are considered, and an
%observation from a subset with the highest value is performed at each
%step. This class of algorithms can be a fruitful field for future
%research.

\subsubsection*{Acknowledgements} 
 
Partially supported by the IMG4 consortium under the MAGNET program, funded
by the Israel Ministry of Trade and Industry, and by the Lynne and William Frankel
center for computer sciences.

\bibliographystyle{plain}
\bibliography{refs}

\end{document}